\documentclass[preprint,11pt]{elsarticle}

\usepackage[T1]{fontenc}
\usepackage{lmodern}
\usepackage{amssymb}
\usepackage{amsmath}
\usepackage{amsthm}
\usepackage{graphicx}
\usepackage{subcaption}
\usepackage{wrapfig}
\usepackage{enumitem}
\usepackage{algorithm}
\usepackage{algpseudocode}
\usepackage{booktabs}
\usepackage[margin=2cm]{geometry}
\usepackage{hyperref}
\usepackage{xurl}
\hypersetup{colorlinks=true,linkcolor=blue!50!black,citecolor=blue!50!black,urlcolor=blue!50!black}

\newtheorem{theorem}{Theorem}
\newtheorem{lemma}{Lemma}
\newtheorem{proposition}{Proposition}
\newtheorem{corollary}{Corollary}
\theoremstyle{definition}
\newtheorem{assumption}{Assumption}
\newtheorem{definition}{Definition}
\newtheorem{remark}{Remark}
\newtheorem{example}{Example}

\makeatletter
\def\ps@pprintTitle{\let\@oddhead\@empty\let\@evenhead\@empty
  \def\@oddfoot{\footnotesize\itshape\hfill\today}\let\@evenfoot\@oddfoot}
\makeatother

\begin{document}

\begin{frontmatter}

\title{Extending Harmonic Function Methods via Root Locus for Coordinated Path Generation}

\author[Pisa]{Matteo Frascarelli}
\ead{matteo.frascarelli@phd.unipi.it}
\author[Piaggio]{Lucia Pallottino}
\ead{lucia.pallottino@unipi.it}
\author[Pisa]{Alberto Landi\corref{cor}}
\ead{alberto.landi@unipi.it}
\cortext[cor]{Corresponding author.}

\affiliation[Pisa]{organization={Department of Information Engineering, University of Pisa},
            city={Pisa}, country={Italy}}
\affiliation[Piaggio]{organization={Research Center E. Piaggio, Department of Information Engineering, University of Pisa},
            city={Pisa}, country={Italy}}

\begin{abstract}
We consider the offline generation of coordinated paths for $N$ interchangeable robots and $N$ goals in planar environments with polygonal obstacles. After a conformal map to a circular domain, starts and goals become the poles and zeros of a rational function $R$, and the paths are the branches of its Generalized Root Locus (GRL), i.e. the streamlines of the logarithmic potential $-\log R$. We show that, for every phase outside a finite set of critical phases and a set of boundary phases, the induced gradient motion drives each robot to a distinct goal along pairwise disjoint paths that never reach the obstacles; the robot--goal assignment is produced by the construction, and with the gain used as a clock all robots move as the closed-loop roots of the GRL and arrive together at a prescribed time. Obstacles become streamlines through successive Milne--Thomson reflections, and for a finite number of reflections a finite test certifies safety for the potential that is actually computed. Reciprocal root-locus planning is recovered as a special case. The method is illustrated on two polygonal workspaces, 100 random instances compared with independent harmonic planning and artificial potentials, and a scene with ten robots and ten obstacles.
\end{abstract}


\end{frontmatter}

\section{Introduction}\label{sec:introduction}

Coordinated motion of several robots in a cluttered environment is usually organized in two layers: an \emph{offline} layer that builds feasible global paths, and an \emph{online} layer that tracks them and handles finite size, actuation limits and close interactions. This paper is about the first layer. We consider $N$ interchangeable robots and $N$ goals in a planar workspace with polygonal obstacles, and look for smooth paths that are collision-free by construction; assigning robots to goals is part of the problem.

Potential-based methods encode the whole problem in a scalar field. The artificial potential field (APF) of Khatib \cite{Khatib_Potential_} may have local minima; harmonic function methods (HFM) \cite{Connolly_Harmonic_Plan} remove them, since solutions of the Laplace equation have no interior extrema. Navigation functions \cite{RimonKoditschek1992} also exclude spurious minima, but need a map onto a sphere world, hard to build for polygonal scenes. For several robots, electrostatic analogies have been used without obstacles \cite{Hyun_Electro_Potentials}, and root-locus (RL) constructions have produced collision-free formations \cite{Coll_free} and single-robot paths among diffusive hazards \cite{Opt_Obs_Using_RL}. The reciprocal root locus (RRL) \cite{Landi,BALESTRINO} handles polygonal workspaces through a conformal map, but enforces obstacle avoidance by heuristic corrections, without a proof of non-penetration.

The Generalized Root Locus (GRL) bridges these lines of work. If starts and goals are the poles and zeros of a rational function $R$, the GRL branches are the streamlines of the logarithmic potential $-\log R$ \cite{On_modeling_RL}: integral curves of a harmonic field, as in HFM, with the branch structure of the root locus. The invariance of the stream function along the gradient motion keeps each robot on its branch, but does not say where the robots go, whether two paths can meet, or whether the obstacles are respected. This paper answers these questions:
\begin{itemize}[leftmargin=*,itemsep=2pt,topsep=3pt]
\item \emph{Certified coordinated motion} (Theorem~\ref{thm:main}). For every phase outside a finite set of critical phases and a set of boundary phases, the GRL in the workspace consists of exactly $N$ disjoint arcs, each joining a start to a different goal. The gradient motion converges along them, the robot--goal assignment is induced by the construction, and with the gain used as a clock the robots leave the starts together, move as the closed-loop roots of the GRL and arrive together at a prescribed time. The result holds on any bounded domain with obstacles.
\item \emph{Certificate for computed obstacle potentials} (Proposition~\ref{prop:certificate}). Obstacles are made streamlines by successive Milne--Thomson reflections. For a finite reflection depth, the boundary residual bounds the set of excluded phases, and a finite test certifies obstacle safety for the potential that is actually computed.
\item \emph{RRL as a special case} (Section~\ref{subsec:rrl}). The RRL planner \cite{Landi,BALESTRINO} is recovered with the same guarantees, which supplies the boundary-invariance argument that was missing.
\item \emph{Algorithm and experiments} (Sections~\ref{sec:procedure} and~\ref{sec:numerics}): the complete procedure, with the critical set verified by the argument principle, evaluated on reference examples, in a Monte Carlo comparison with independent harmonic planning and artificial potentials, and on a scene with ten robots and ten obstacles.
\end{itemize}

Code and an interactive companion notebook are available online (see \emph{Code availability}).

\section{Problem formulation and GRL potential}\label{sec:problem}

\subsection{Workspace and canonical domain}\label{subsec:workspace}

Let $P\subset\mathbb{C}$ be a bounded workspace whose boundary is the union of $N_{\mathrm{obs}}$ pairwise disjoint simple closed polygons $\Gamma_0,\dots,\Gamma_{N_{\mathrm{obs}}-1}$, with $\Gamma_0$ the outer one. At the planning level the robots are interchangeable point agents. $N$ initial positions and $N$ goals are given in $P$. For robots of radius $\varrho$, a margin is introduced by inflating every obstacle and shrinking the outer boundary by $\varrho$ before planning. This handles the clearance between robots and environment only. Mutual safety is not implied by disjoint paths and requires the distance between distinct paths to exceed $2\varrho$; this condition does not depend on time, is checked after planning and can be included in the choice of the phase.

Following \cite{Landi,BALESTRINO}, the workspace is mapped onto a circular domain by a conformal map
\begin{equation}\label{eq:conf_map}
z=f(s),\qquad f:P\to D,
\end{equation}
where $D$ is the disk bounded by a circle $C_0$ minus $N_{\mathrm{obs}}-1$ closed disks bounded by circles $C_i=\{|z-c_i|=r_i\}$, with $f(\Gamma_i)=C_i$. Schwarz--Christoffel formulas are semi-explicit only for simply connected polygons \cite{Driscoll_SC_book,DeLillo_MC_SC}; we compute $f$ and $f^{-1}$ with the toolbox \texttt{PlgCirMap} \cite{Nasser_toolbox}, based on Nasser's fast Koebe iteration \cite{Nasser_fastcomp}, with cost $O(N_{\mathrm{obs}}^2N_\Gamma+N_{\mathrm{obs}}N_\Gamma\log N_\Gamma)$ for $N_\Gamma$ nodes per boundary component. The map is normalized by an interior point $\alpha\in P$ sent to the origin, taken in a central free region (Fig.~\ref{fig:conf_reduction}).

All the analysis is carried out in $D$, and the paths are mapped back by $f^{-1}$. The images in $D$ of starts and goals are still denoted by $p_1,\dots,p_N$ and $\zeta_1,\dots,\zeta_N$. The planning problem in $D$ is to find $N$ smooth, pairwise disjoint curves in $D$ joining the poles to the zeros, with a pairing produced by the construction itself.

\begin{figure}[tb]
\centering
\begin{subfigure}{0.3\textwidth}
    \centering
    \includegraphics[width=\linewidth]{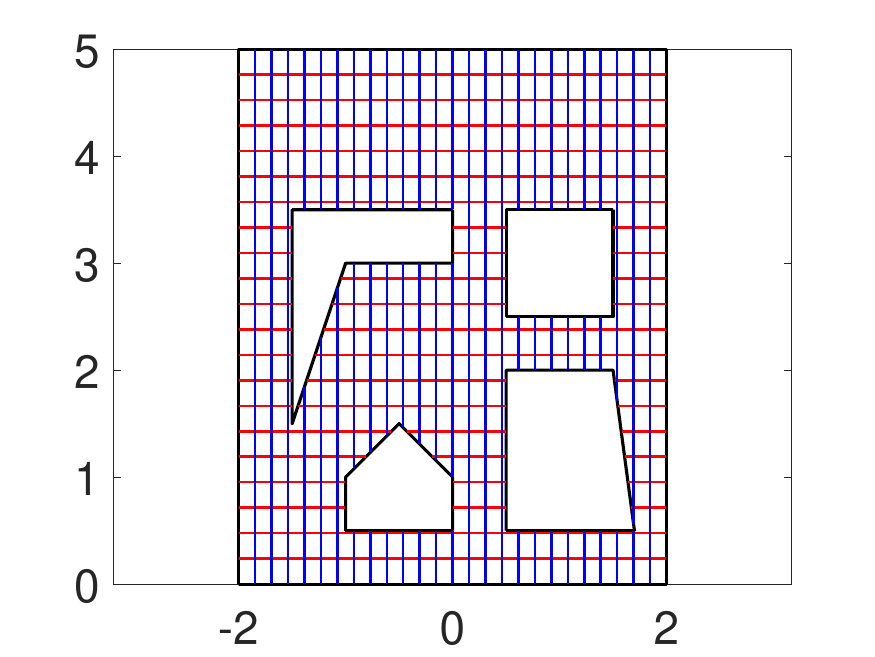}
    \caption{Workspace $P$}
\end{subfigure}
\hspace{0.06\textwidth}
\begin{subfigure}{0.3\textwidth}
    \centering
    \includegraphics[width=\linewidth]{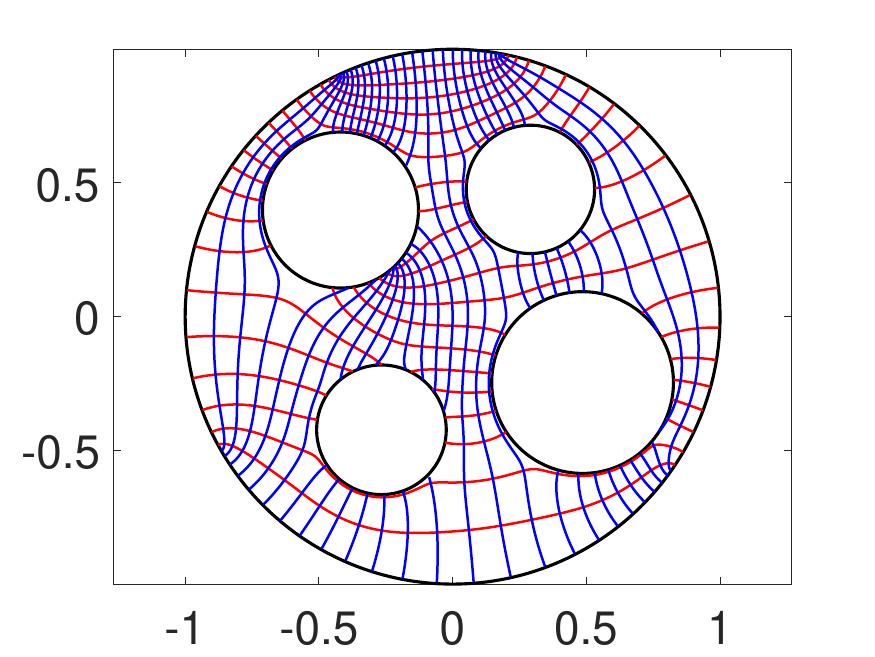}
    \caption{Circular domain $D$}
\end{subfigure}
\caption{Conformal map of the polygonal workspace of the Flagship example onto a circular domain.}
\label{fig:conf_reduction}
\end{figure}

\subsection{Complex potential and Generalized Root Locus}\label{subsec:potential}

Starts and goals are encoded by the rational function
\begin{equation}\label{eq:R}
R(z)=\frac{n(z)}{d(z)}=\frac{\prod_{i=1}^{N}(z-\zeta_i)}{\prod_{i=1}^{N}(z-p_i)},
\end{equation}
where $n$ and $d$ are monic, coprime and have simple roots. Let $\Sigma:=\{p_1,\dots,p_N,\zeta_1,\dots,\zeta_N\}$ and $j$ the imaginary unit. The complex potential is
\begin{equation}\label{eq:Omega}
\Omega(z):=-\log R(z)=\Phi(z)+j\Psi(z),\qquad z\notin\Sigma .
\end{equation}
The scalar potential $\Phi=-\log|R|$ is single valued and harmonic outside $\Sigma$. The stream function $\Psi=-\arg R$ is defined modulo $2\pi$; a holomorphic branch of $\log R$ exists on every simply connected subset of $\mathbb{C}\setminus\Sigma$ \cite{ahlfors1979}, and the gradient of $\Psi$ is single valued everywhere outside $\Sigma$. This field-theoretic reading of the root locus goes back to \cite{On_modeling_RL}.

Identifying $z=x+jy$ with $(x,y)^\top\in\mathbb{R}^2$, $\nabla$ denotes the gradient, written as a row vector, and $\Delta$ the Laplacian. By the Cauchy--Riemann relations,
\begin{equation}\label{eq:CR}
\Delta\Phi=\Delta\Psi=0,\qquad \nabla\Phi\,\nabla\Psi^\top=0,\qquad \|\nabla\Phi\|=\|\nabla\Psi\|=|\Omega'|,
\end{equation}
and $\nabla\Phi^\top$, read as a complex number, equals $\overline{\Omega'}$. Hence $\Phi$ plays the role of a scalar potential and $\Psi$ of the associated stream function. For a set $A$, $\overline A$ and $\partial A$ denote its closure and boundary; for a number or a function, the bar denotes complex conjugation.

For a phase $\theta\in[-\pi,\pi)$, the \emph{Generalized Root Locus} is
\begin{equation}\label{eq:GRL}
\mathrm{GRL}(\theta):=\{z\notin\Sigma:\ 1+ke^{j\theta}R(z)=0\ \text{for some } k>0\}.
\end{equation}
It extends the classical root locus \cite{Evans_RL,ogata2002modern}: $\mathrm{GRL}(0)$ and $\mathrm{GRL}(-\pi)$ are the positive and the negative root loci. For a fixed $\theta$ we call $\mathrm{GRL}(\theta)$ a \emph{branch family}, and its connected components \emph{branches}.

\begin{lemma}\label{lem:phase}
$z\in\mathrm{GRL}(\theta)$ if and only if $\Psi(z)\equiv\theta+\pi \pmod{2\pi}$, and then $k=e^{\Phi(z)}$.
\end{lemma}
\begin{proof}
$1+ke^{j\theta}R(z)=0$ is equivalent to $R(z)=-e^{-j\theta}/k$. Taking modulus and argument gives $e^{-\Phi(z)}=1/k$ and $-\Psi(z)\equiv\pi-\theta$.
\end{proof}

So the GRL is a level set of the stream function, and the gain is a function of the position, given by the scalar potential.

\subsection{Gradient-induced motion}\label{subsec:motion}

Each robot follows the gradient of the scalar potential,
\begin{equation}\label{eq:motion}
\dot z=\lambda(z)\nabla\Phi(z)^\top=-\lambda(z)\,\overline{R'(z)/R(z)},
\end{equation}
where $\lambda>0$ is a locally Lipschitz speed modulation. Since $\lambda>0$, it changes the time law but not the integral curves.

\begin{lemma}\label{lem:invariance}
For every $z_0\notin\Sigma$, \eqref{eq:motion} has a unique maximal solution with $z(0)=z_0$. Along it, $\Psi(z(t))$ is constant modulo $2\pi$, and
\begin{equation}\label{eq:phi_dot}
\tfrac{d}{dt}\Phi(z(t))=\lambda(z(t))\|\nabla\Phi(z(t))\|^2\ge0,
\end{equation}
with strict inequality where $\nabla\Phi\neq0$. Hence a solution starting on $\mathrm{GRL}(\theta)$ remains on $\mathrm{GRL}(\theta)$, with nondecreasing gain $k(t)=e^{\Phi(z(t))}$.
\end{lemma}
\begin{proof}
$\Phi$ is smooth outside $\Sigma$ and $\lambda$ is locally Lipschitz, so the vector field is locally Lipschitz and the maximal solution exists and is unique. By the chain rule and \eqref{eq:CR}, $\tfrac{d}{dt}\Psi=\lambda\nabla\Psi\nabla\Phi^\top=0$ and $\tfrac{d}{dt}\Phi=\lambda\|\nabla\Phi\|^2$. The last claim follows from Lemma~\ref{lem:phase}.
\end{proof}

Lemma~\ref{lem:invariance} is a local statement: it does not say where a trajectory goes, whether it can reach a break point where two branches meet, or whether it can leave the workspace.

\section{Break points and critical phases}\label{sec:breakpoints}

Distinct branches of the same family can meet only at break points, usually defined by a multiplicity condition \cite{Coll_free}; here they have a direct interpretation in terms of the potential.

\begin{definition}[Break point]\label{def:bp}
A point $b\notin\Sigma$ is a \emph{break point} if, for some $\theta$ and $k>0$, it is a root of multiplicity at least two of $1+ke^{j\theta}R(z)=0$.
\end{definition}

\begin{proposition}[Critical characterization]\label{prop:critical}
For $b\notin\Sigma$ the conditions $\nabla\Phi(b)=0$, $\nabla\Psi(b)=0$, $\Omega'(b)=0$ and $R'(b)=0$ are equivalent. Every break point satisfies them, and every point satisfying them is a break point of the family $\theta=\pi-\arg R(b)$.
\end{proposition}
\begin{proof}
The first three conditions are equivalent by \eqref{eq:CR}. Differentiating \eqref{eq:Omega} gives $\Omega'=-R'/R$, and $R(b)\neq0$, so $\Omega'(b)=0$ if and only if $R'(b)=0$. If $b$ is a multiple root of $1+ke^{j\theta}R$, differentiating at $b$ gives $ke^{j\theta}R'(b)=0$, hence $R'(b)=0$. Conversely, if $R'(b)=0$, take $\theta=\pi-\arg R(b)$ and $k=1/|R(b)|$: then $1+ke^{j\theta}R(b)=0$ and the derivative vanishes too, so $b$ is a multiple root.
\end{proof}

Break points are therefore the critical points of the complex potential: the stagnation points of the flow. In the obstacle-free case their number is bounded, as shown in \cite{Coll_free}; we restate the bound in the present notation.

\begin{proposition}[Break-point bound]\label{prop:bp_bound}
For $R$ as in \eqref{eq:R}, the number of break points is at most $2N-2$.
\end{proposition}
\begin{proof}
Let $q:=n'd-d'n$, so that $R'=q/d^2$. By Proposition~\ref{prop:critical}, every break point is a zero of $q$ outside $\Sigma$. For $N=1$, $q=\zeta_1-p_1\neq0$ and there is no break point. For $N\ge2$, $n'd$ and $d'n$ have degree $2N-1$ and the same leading coefficient $N$, because $n$ and $d$ are monic. Hence $\deg q\le2N-2$.
\end{proof}

With the image singularities of Section~\ref{sec:obstacles} the bound no longer applies, and the number of critical points in the domain is verified by the argument principle (Section~\ref{sec:procedure}). A phase is \emph{critical} if its family contains a break point. By Proposition~\ref{prop:critical}, the set of critical phases associated with a domain $D$ is
\begin{equation}\label{eq:Theta_b}
\Theta_b:=\{\pi-\arg R(b)\!\!\pmod{2\pi}:\ b\in D\setminus\Sigma,\ R'(b)=0\}.
\end{equation}
Each critical point gives one phase, and the critical points are finitely many in $\overline D$ when $R$ is meromorphic near $\overline D$, so $\Theta_b$ is finite. The critical points do not depend on $\theta$, so $\Theta_b$ is computed once.

\section{Certified coordinated GRL motion}\label{sec:main}

The main result is stated on a general bounded domain, so that the obstacle-free, the obstacle-compatible and the reciprocal constructions are covered at once.

\begin{assumption}\label{ass:domain}
$D\subset\mathbb{C}$ is a bounded domain whose boundary is a finite union of disjoint piecewise-smooth Jordan curves. $R$ is meromorphic on a neighbourhood of $\overline D$, and its zeros and poles in $\overline D$ are exactly $\zeta_1,\dots,\zeta_N$ and $p_1,\dots,p_N$, all simple and contained in $D$.
\end{assumption}

The assumption allows $R$ to have further zeros and poles outside $\overline D$, as the images of Section~\ref{sec:obstacles} do. It also fixes what happens at the boundary: $R$ is finite and nonzero on $\partial D$. Let
\begin{equation}\label{eq:F_kappa}
F_\theta:=e^{j\theta}R,\qquad \kappa(z):=e^{\Phi(z)}=\frac{1}{|R(z)|},\quad z\in D\setminus\Sigma .
\end{equation}
Then $\mathrm{GRL}(\theta)\cap D=\{z\in D\setminus\Sigma:F_\theta(z)<0\}$: the GRL is the set where $F_\theta$ takes negative real values. On $\mathrm{GRL}(\theta)$ one has $\kappa=-1/F_\theta$, so $\kappa(z)$ is the gain that places a closed-loop root at $z$. Besides the critical phases \eqref{eq:Theta_b}, the phases whose family touches the boundary are excluded. By Lemma~\ref{lem:phase} they are the \emph{boundary phases}
\begin{equation}\label{eq:Theta_d}
\Theta_\partial:=\{\Psi(z)-\pi\!\!\pmod{2\pi}:\ z\in\partial D\}.
\end{equation}
Since $\Psi$ varies continuously along $\partial D$, each boundary component contributes an arc of phases; it contributes a single phase exactly when it is a streamline. A phase $\theta\notin\Theta_b\cup\Theta_\partial$ is called \emph{admissible}. Each connected component of $\mathrm{GRL}(\theta)\cap D$ is called a \emph{branch}.

\begin{theorem}[Certified coordinated GRL motion]\label{thm:main}
Let Assumption~\ref{ass:domain} hold and let $\theta$ be admissible. Then:
\begin{enumerate}[leftmargin=*,itemsep=2pt,topsep=3pt]
\item \emph{(Branch structure.)} $\mathrm{GRL}(\theta)\cap D$ is the disjoint union of exactly $N$ smooth open arcs $B_1,\dots,B_N$. Each $B_i$ joins $p_i$ to $\zeta_{\sigma(i)}$ for a permutation $\sigma=\sigma_\theta$ of $\{1,\dots,N\}$, $\kappa$ maps $B_i$ diffeomorphically onto $(0,\infty)$, and $\overline{B_i}\cap\partial D=\emptyset$.
\item \emph{(Convergence.)} Let $\lambda_i>0$ be locally Lipschitz near $B_i$. The maximal solution of $\dot z_i=\lambda_i(z_i)\nabla\Phi(z_i)^\top$ with $z_i(0)\in B_i$ remains in $B_i$, $\kappa(z_i(t))$ is strictly increasing and unbounded, and $z_i(t)\to\zeta_{\sigma(i)}$.
\item \emph{(Safety.)} For $i\neq j$ and all $t,\tau$ in the respective intervals of existence, $z_i(t)\neq z_j(\tau)$ and $z_i(t)\notin\partial D$.
\item \emph{(Synchronous motion from the starts.)} Let $k:[0,T)\to[0,\infty)$ be $C^2$ with $\dot k>0$, with $k(0)=0$ and $k(t)\to\infty$ as $t\to T$, e.g. $k(t)=\tan(\pi t/2T)$. Then $z_i(t):=(\kappa|_{B_i})^{-1}(k(t))$, with $z_i(0)=p_i$, solves \eqref{eq:motion} on $(0,T)$ with $\lambda_i=\dot k/(\kappa\|\nabla\Phi\|^2)$, and $z_i(t)\to\zeta_{\sigma(i)}$ as $t\to T$. At every $t$ the robots occupy the $N$ roots in $D$ of $1+k(t)e^{j\theta}R(z)=0$: they leave the starts together and reach the goals together at $T$. For $k(t)=\tan(\pi t/2T)$ the speed $\|\dot z_i\|=\dot k/(\kappa\|\nabla\Phi\|)$ is continuous and positive on $[0,T]$, equal to $\tfrac{\pi}{2T}|\rho_i|$ at the start and to $\tfrac{\pi}{2T}/|R'(\zeta_{\sigma(i)})|$ at the goal, where $\rho_i$ is the residue of $R$ at $p_i$.
\end{enumerate}
\end{theorem}

\emph{Strategy of the proof.} The level set of $\Psi$ followed by a robot could a priori reach the boundary, pass through a break point, or lead two robots to the same goal. Admissibility rules these out one at a time: $\theta\notin\Theta_b$ makes every branch a regular arc (Step~1); compactness of $\overline D$ and $\theta\notin\Theta_\partial$ force every branch to end at a zero (Step~2); and the local inversion of $F_\theta$ at simple poles and zeros, where exactly one arc starts or ends, makes the pairing a bijection (Step~3). Step~4 adds the dynamics.

\begin{proof}
\emph{Step 1 (local structure).} Let $z^\ast\in\mathrm{GRL}(\theta)\cap D$. Since $\theta\notin\Theta_b$, $z^\ast$ is not a critical point of $R$ (Proposition~\ref{prop:critical}), so $F_\theta'(z^\ast)=e^{j\theta}R'(z^\ast)\neq0$. By the holomorphic inverse function theorem, $F_\theta$ maps a neighbourhood $U$ of $z^\ast$ biholomorphically onto a neighbourhood $V$ of $F_\theta(z^\ast)<0$. With $G:V\to U$ the local inverse, $\mathrm{GRL}(\theta)\cap U=G(V\cap\mathbb{R}_{<0})$ is the image of a segment under an injective map with $G'=1/F_\theta'\neq0$, that is, a single smooth arc. Along it $\kappa=-1/F_\theta$ is a smooth coordinate, hence strictly monotone. Thus $\mathrm{GRL}(\theta)\cap D$ is an embedded one-dimensional submanifold of $D\setminus\Sigma$: distinct components never meet, and $\kappa$ has no critical point on any of them. A connected one-dimensional manifold is diffeomorphic to a circle or to an open interval. A circle is excluded, because the continuous function $\kappa$ would attain a maximum on it, at a critical point. Hence every component $B$ is an open arc on which $\kappa$ is strictly monotone, and $\kappa(B)=(\kappa_-,\kappa_+)$ with $0\le\kappa_-<\kappa_+\le\infty$.

\emph{Step 2 (ends).} Let $z_n\in B$ with $\kappa(z_n)\to\kappa_+$. Since $\overline D$ is compact, up to a subsequence $z_n\to z^\ast\in\overline D$. Three cases are possible. If $z^\ast\in\partial D$, then $R(z^\ast)$ is finite and nonzero by Assumption~\ref{ass:domain}, and $F_\theta(z^\ast)=\lim F_\theta(z_n)$ is a negative real number, so $z^\ast$ is a boundary point of $\mathrm{GRL}(\theta)$ and $\theta\in\Theta_\partial$, a contradiction. If $z^\ast\in D\setminus\Sigma$, then $\kappa_+=\kappa(z^\ast)<\infty$ and $z^\ast\in\mathrm{GRL}(\theta)$. By Step~1 the arc through $z^\ast$ is the only part of $\mathrm{GRL}(\theta)$ near $z^\ast$, so it contains $z_n$ for large $n$ and is a subset of $B$. Then $\kappa_+=\kappa(z^\ast)$ is attained inside the open interval $(\kappa_-,\kappa_+)$, a contradiction. If $z^\ast$ is a pole, then $\kappa(z_n)=1/|R(z_n)|\to0\neq\kappa_+$. Therefore $\kappa_+=\infty$, and every cluster point of $B$ as $\kappa\to\infty$ is a zero of $R$. The set of these cluster points is connected (it is the intersection of a decreasing family of compact connected sets) and contained in the finite set $\{\zeta_\ell\}$, so it is a single zero. In the same way, $\kappa_-=0$ and $B$ starts at a single pole. In particular $\overline B=B\cup\{p,\zeta\}$ for a pole $p$ and a zero $\zeta$, and $\overline B\cap\partial D=\emptyset$.

\emph{Step 3 (counting).} Near a simple pole $p_i$, $w:=1/F_\theta$ is holomorphic with $w(p_i)=0$ and $w'(p_i)\neq0$. Hence $w$ is a local biholomorphism at $p_i$, and $\mathrm{GRL}(\theta)$ near $p_i$ is $w^{-1}(\mathbb{R}_{<0})$, the image of a single half-line: exactly one arc ends at $p_i$. The same holds at every simple zero, with $F_\theta$ in place of $w$. By Step~2 every component has exactly one pole end and one zero end. Each pole is the end of exactly one component and each zero is the end of exactly one component, so there are exactly $N$ components and the map $\sigma$ sending the pole end to the zero end is a bijection. This proves item~1.

\emph{Step 4 (dynamics).} On $\mathrm{GRL}(\theta)\cap D$, $\nabla\Phi^\top=\overline{\Omega'}\neq0$ by Step~1, so the vector field of item~2 is locally Lipschitz near $B_i$. By Lemma~\ref{lem:invariance} the solution stays in $\mathrm{GRL}(\theta)$; being continuous and starting in $B_i$, it stays in the component $B_i$. Along it $\tfrac{d}{dt}\kappa=\kappa\,\tfrac{d}{dt}\Phi=\kappa\lambda_i\|\nabla\Phi\|^2>0$. Using $\kappa$ as a coordinate on $B_i\cong(0,\infty)$, the motion reduces to a scalar equation $\dot\kappa=g_i(\kappa)$ with $g_i>0$ locally Lipschitz. If $\kappa$ stayed bounded on the maximal interval $[0,T_{\max})$, $z_i(t)$ would remain in a compact subarc of $B_i$ on which $g_i\ge c>0$: for $T_{\max}=\infty$ this gives $\kappa\ge\kappa(0)+ct\to\infty$, and for $T_{\max}<\infty$ the solution could be continued beyond $T_{\max}$ from a regular point of $B_i$; both are impossible. Hence $\kappa\to\infty$ and, by item~1, $z_i\to\zeta_{\sigma(i)}$. This proves item~2. Item~3 follows from the disjointness of the $B_i$ and from $\overline{B_i}\cap\partial D=\emptyset$.

For item~4, by item~1 the point $z_i(t)=(\kappa|_{B_i})^{-1}(k(t))$ is well defined and $C^1$ on $(0,T)$, and it tends to $p_i$ as $k\to0$ and to $\zeta_{\sigma(i)}$ as $k\to\infty$. Along \eqref{eq:motion}, $\tfrac{d}{dt}\kappa=\kappa\lambda_i\|\nabla\Phi\|^2$, which equals $\dot k$ for the stated $\lambda_i$; since the trajectory is tangent to $B_i$ and $\kappa$ is a coordinate on $B_i$, $z_i(t)$ is the solution of \eqref{eq:motion}. All robots share the gain $k(t)$, and each $B_i$ contains exactly one point with $\kappa=k(t)$. Since $\mathrm{GRL}(\theta)\cap D=\bigcup_iB_i$, by Lemma~\ref{lem:phase} these $N$ points are all the roots in $D$ of $1+k(t)e^{j\theta}R=0$. The speed is $\lambda_i\|\nabla\Phi\|=\dot k/(\kappa\|\nabla\Phi\|)$. Near $p_i$, $R(z)\simeq\rho_i/(z-p_i)$, so $\kappa\simeq|z-p_i|/|\rho_i|$ and $\|\nabla\Phi\|=|R'/R|\simeq1/|z-p_i|$, hence $\kappa\|\nabla\Phi\|\to1/|\rho_i|$. Near $\zeta=\zeta_{\sigma(i)}$, $R(z)\simeq R'(\zeta)(z-\zeta)$, so $\kappa\simeq1/(|R'(\zeta)||z-\zeta|)$ and $\|\nabla\Phi\|\simeq1/|z-\zeta|$, hence $\kappa\|\nabla\Phi\|\simeq|R'(\zeta)|\kappa^2$. With $\dot k=\tfrac{\pi}{2T}(1+k^2)$ this gives the two limits, and in between the speed is continuous and positive because $\nabla\Phi\neq0$ on $B_i$.
\end{proof}

Theorem~\ref{thm:main} turns the invariance of Lemma~\ref{lem:invariance} into a complete multi-robot statement. Each robot reaches a goal, the goals are distinct, the assignment $\sigma_\theta$ is produced rather than prescribed, no two paths meet, and $\partial D$ is never reached. Safety in item~3 holds for all pairs of times $t,\tau$, not only for equal times: the paths themselves are disjoint, so any speed modulation, chosen independently for each robot, is collision free. The distances between paths and from $\partial D$ are positive but not prescribed; if needed, they can be improved by the choice of the phase. Item~4 gives the root-locus reading of the motion: the robots are the closed-loop poles of $1+ke^{j\theta}R$, and the time law is a reparametrization of the gain.

The direction in which $B_i$ leaves $p_i$ follows from $R(z)\simeq\rho_i/(z-p_i)$ and the phase condition:
\begin{equation}\label{eq:departure}
\arg(z-p_i)\to\pi+\theta+\arg\rho_i\quad\text{along }B_i .
\end{equation}
This is the departure angle of the root locus, used to start the numerical tracing.

\begin{remark}[Assignment and phase]\label{rem:sigma}
The assignment $\sigma_\theta$ is constant on every interval of admissible phases. Indeed, on such an interval the branches depend continuously on $\theta$, while their end points are fixed and the branches cannot touch each other or the boundary. The assignment can therefore change only across $\Theta_b\cup\Theta_\partial$. Choosing the phase is thus choosing one of finitely many assignments, together with the shape of the paths that realize it.
\end{remark}

\begin{remark}[Obstacle-free plane]\label{rem:plane}
Without obstacles the workspace is unbounded, but Theorem~\ref{thm:main} applies to any disk $D_\rho=\{|z|<\rho\}$ containing $\Sigma$. Since $N$ zeros and $N$ poles give $R(z)\to1$ as $z\to\infty$, $\Psi\to0$ uniformly on $|z|=\rho$, and $\Theta_\partial$ shrinks to the single phase $-\pi$ as $\rho\to\infty$. This is the familiar fact that only branches of the negative root locus can escape to infinity. Every phase $\theta\notin\Theta_b\cup\{-\pi\}$ is admissible on a large enough disk.
\end{remark}

\begin{wrapfigure}{r}{0.4\textwidth}
\vspace{-12pt}
\centering
\includegraphics[width=0.4\textwidth]{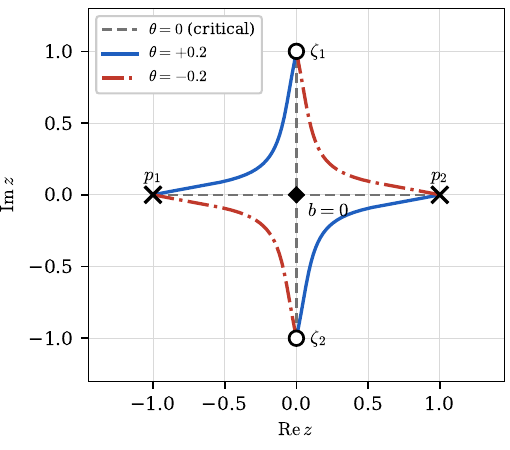}
\caption{Two-robot example: GRL branches for $\theta=0.2$ (solid) and $\theta=-0.2$ (dash-dotted), and the critical family $\theta=0$ (dashed) through the break point $b=0$. The assignment is exchanged across the critical phase. Crosses: starts; circles: goals.}
\label{fig:two_robots}
\vspace{-10pt}
\end{wrapfigure}
\begin{example}[Two robots]\label{ex:two}
Let $p_{1,2}=\mp1$ and $\zeta_{1,2}=\pm j$, so that $R(z)=(z^2+1)/(z^2-1)$ and $R'(z)=-4z/(z^2-1)^2$. The only critical point is $b=0$, with $R(0)=-1$, hence $\Theta_b=\{0\}$: the classical root locus is the only critical family. For $\theta=0$ the closed-loop roots satisfy $z^2=(1-k)/(1+k)$, so both robots reach the origin at $k=1$ and the assignment is undefined. For $\theta\neq0$, $z^2=(1-ke^{j\theta})/(1+ke^{j\theta})$ never vanishes, the two branches are disjoint, and $\sigma_\theta$ is the identity for $\theta>0$ and the transposition for $\theta<0$ (Fig.~\ref{fig:two_robots}). The residues are $\rho_{1,2}=\mp1$ and $|R'(\pm j)|=1$, so with $k(t)=\tan(\pi t/2T)$ both robots start and arrive with speed $\pi/2T$.
\end{example}

\section{Obstacle-compatible potentials and certificate}\label{sec:obstacles}

In the circular domain
\begin{equation}\label{eq:circular_domain}
D=\{|z-c_0|<r_0\}\setminus\bigcup_{i=1}^{N_{\mathrm{obs}}-1}\{|z-c_i|\le r_i\},
\end{equation}
the potential \eqref{eq:Omega} does not respect the obstacles: $\Psi$ is not constant on the circles, so each circle contributes to $\Theta_\partial$ a wide arc of phases, and few admissible phases may remain.

\subsection{Boundary conditions}\label{subsec:bc}

Let $n_i$ and $\tau_i$ denote the unit normal and tangent to $C_i$, $i=0,\dots,N_{\mathrm{obs}}-1$. We look for a potential $\Omega_{\mathrm{obs}}=\Phi_{\mathrm{obs}}+j\Psi_{\mathrm{obs}}$ with the same singularities as $\Omega$ in $D$ and
\begin{equation}\label{eq:neumann}
\Delta\Phi_{\mathrm{obs}}=0\ \text{in }D\setminus\Sigma,\qquad \frac{\partial\Phi_{\mathrm{obs}}}{\partial n_i}=0\ \text{on }C_i .
\end{equation}
The first condition keeps the harmonic structure; the second asks that the field have no normal component on the obstacles. By the Cauchy--Riemann relations,
\begin{equation}\label{eq:tangential}
\frac{\partial\Psi_{\mathrm{obs}}}{\partial\tau_i}=\frac{\partial\Phi_{\mathrm{obs}}}{\partial n_i}\quad\text{on }C_i,
\end{equation}
so the homogeneous Neumann condition is equivalent to $\Psi_{\mathrm{obs}}$ being constant on each $C_i$: every boundary circle is a streamline \cite{milne-thomson1952theoretical,CROWDY2_analytic_multi_circles,DELILLO_pot_flow_series}. In terms of phases, each circle then contributes a single phase to $\Theta_\partial$. The condition gives non-penetration only; clearance margins are introduced beforehand by inflating the obstacles (Section~\ref{subsec:workspace}).

\subsection{One circle: the Milne--Thomson correction}\label{subsec:mt}

Let $C=\{|z-c|=r\}$ be one of the circles, and let
\begin{equation}\label{eq:reflection}
\rho_C(z):=c+\frac{r^2}{\overline{z-c}}
\end{equation}
be the reflection in $C$. It fixes $C$ pointwise, exchanges the inside and the outside of the circle, and is an involution.

\begin{theorem}[Milne--Thomson circle theorem \cite{milne-thomson1952theoretical}]\label{thm:mt}
Let $\Omega$ be holomorphic in the flow region, up to its singularities, and free of singularities on the other side of $C$. Then
\begin{equation}\label{eq:mt}
\Omega_C(z):=\Omega(z)+\overline{\Omega(\rho_C(z))}
\end{equation}
has the same singularities as $\Omega$ in the flow region, and $C$ is a streamline of $\Omega_C$.
\end{theorem}
\begin{proof}
For $z$ in the flow region, $\rho_C(z)$ lies on the other side of $C$, where $\Omega$ has no singularities. The composition of a holomorphic function with the antiholomorphic map $\rho_C$, followed by conjugation, is holomorphic, so the added term is holomorphic in the flow region and no singularity is created there. If $z\in C$, then $\rho_C(z)=z$ and $\Omega_C(z)=\Omega(z)+\overline{\Omega(z)}=2\Re\Omega(z)$ is real. Hence $\Im\Omega_C$ is constant on $C$.
\end{proof}

In the logarithmic setting, $\overline{\Omega(\rho_C(z))}$ adds an image pole at $\rho_C(p)$ for each pole $p$ and an image zero for each zero, plus terms $\pm\log(z-c)$ that cancel in the balanced case. So
\begin{equation}\label{eq:mt_rational}
R_C(z)=K\,R(z)\,\frac{\prod_{i}\bigl(z-\rho_C(\zeta_i)\bigr)}{\prod_{i}\bigl(z-\rho_C(p_i)\bigr)},\qquad K\neq0,
\end{equation}
is again rational and balanced, with image singularities inside the obstacle only. A positive factor in $K$ only rescales the gain.

\begin{example}[One robot in a disk]\label{ex:disk}
For $p=-1$, $\zeta=1$ and $D=\{|z|<2\}$, $\arg R(z)$ is the angle under which the segment $[-1,1]$ is seen from $z$, so each $\mathrm{GRL}(\theta)$ is a circular arc from $p$ to $\zeta$. With $\Omega=-\log R$, the arcs with $|\theta|\ge\pi-2\arctan\tfrac12\approx126.9^\circ$ leave the disk, and $\Theta_\partial=\{|\theta|\ge126.9^\circ\}$. One reflection in $C_0=\{|z|=2\}$ adds the images $\rho_{C_0}(\pm1)=\pm4$, makes $C_0$ a streamline, and reduces $\Theta_\partial$ to the single phase $-\pi$.
\end{example}

\subsection{Several circles: successive reflections}\label{subsec:reflections}

With several circles, the images created to correct one circle perturb the others, and the correction is iterated. Let $P^{(0)}=\{p_i\}$ and $Z^{(0)}=\{\zeta_i\}$. The first reflected sets are $P_i^{(1)}:=\rho_{C_i}(P^{(0)})$ and $Z_i^{(1)}:=\rho_{C_i}(Z^{(0)})$, $i=0,\dots,N_{\mathrm{obs}}-1$. For $s\ge2$,
\begin{equation}\label{eq:recursive}
P_i^{(s)}:=\bigcup_{\ell\neq i}\rho_{C_i}\bigl(P_\ell^{(s-1)}\bigr),\qquad Z_i^{(s)}:=\bigcup_{\ell\neq i}\rho_{C_i}\bigl(Z_\ell^{(s-1)}\bigr):
\end{equation}
the images generated on $C_\ell$ at depth $s-1$ are reflected in every other circle. Reflecting again in the same circle would only undo the previous step. This is the standard method of images for multiply connected circular domains \cite{CROWDY2_analytic_multi_circles,DELILLO_comp_rev,DELILLO_pot_flow_series}. By induction, the number of image singularities up to depth $m$ is
\begin{equation}\label{eq:NI}
N_I^{(m)}=2NN_{\mathrm{obs}}\textstyle\sum_{s=0}^{m-1}(N_{\mathrm{obs}}-1)^s ,
\end{equation}
and the truncated function is
\begin{equation}\label{eq:R_obs}
R^{(m)}_{\mathrm{obs}}(z)=K_m\,R(z)\prod_{s=1}^{m}\prod_{i=0}^{N_{\mathrm{obs}}-1}\frac{\prod_{\zeta\in Z_i^{(s)}}(z-\zeta)}{\prod_{p\in P_i^{(s)}}(z-p)},\qquad \Omega^{(m)}_{\mathrm{obs}}=-\log R^{(m)}_{\mathrm{obs}} .
\end{equation}
Each image lies inside an obstacle disk or outside $C_0$, because a reflection in $C_i$ sends the outside of the disk bounded by $C_i$ into it (and, for $C_0$, the inside of $C_0$ outside it), and every reflected point was outside the reflecting circle. Hence $R^{(m)}_{\mathrm{obs}}$ is rational, has exactly the original zeros and poles in $\overline D$, and all images outside $\overline D$. \emph{It satisfies Assumption~\ref{ass:domain}, and Theorem~\ref{thm:main} applies exactly to the potential that is computed, not to its limit for $m\to\infty$.} Figure~\ref{fig:flagship}(a) shows the resulting field for the Flagship example.

The number of images grows geometrically with $m$. Convergence of the infinite series is known under separation conditions on the circles \cite{DELILLO_pot_flow_series,DELILLO_comp_rev}; we do not rely on it, and certify a finite truncation directly.

\subsection{Residual and certificate}\label{subsec:certificate}

For finite $m$, $\Psi^{(m)}_{\mathrm{obs}}$ is only approximately constant on $C_i$. The residual
\begin{equation}\label{eq:residual}
\bar\varepsilon_i:=\max_{z\in C_i}\left|\frac{\partial\Phi^{(m)}_{\mathrm{obs}}}{\partial n_i}(z)\right|
\end{equation}
is the largest normal component of the field $\nabla\Phi^{(m)}_{\mathrm{obs}}$ on $C_i$, that is, how much the computed field still enters the obstacle. By \eqref{eq:tangential} it is also the largest slope of $\Psi^{(m)}_{\mathrm{obs}}$ along $C_i$, so it controls how many phases each circle removes.

\begin{proposition}[Residual bound and certificate]\label{prop:certificate}
Let $R=R^{(m)}_{\mathrm{obs}}$ and $\Psi=\Psi^{(m)}_{\mathrm{obs}}$. Then:
\begin{enumerate}[leftmargin=*,itemsep=2pt,topsep=3pt]
\item The contribution of $C_i$ to $\Theta_\partial$ lies in an arc of length at most $\pi r_i\bar\varepsilon_i$. Hence $\Theta_\partial$ has total length at most $\pi\sum_ir_i\bar\varepsilon_i$, which vanishes with the residual.
\item Let $z_{i,1},\dots,z_{i,M_i}$ be $M_i$ equispaced points on $C_i$, $h_i=2\pi r_i/M_i$, and
\begin{equation}\label{eq:Li}
L_i:=\sum_{\xi}\frac{1}{\mathrm{dist}(\xi,C_i)},
\end{equation}
where $\xi$ ranges over all zeros and poles of $R$, images included. If, for every $i$,
\begin{equation}\label{eq:certificate}
\min_{\ell}\ d_{2\pi}\bigl(\Psi(z_{i,\ell}),\theta+\pi\bigr)>L_ih_i/2,
\end{equation}
where $d_{2\pi}$ is the distance modulo $2\pi$, then $\theta\notin\Theta_\partial$.
\end{enumerate}
\end{proposition}
\begin{proof}
Each reflection maps $N$ poles and $N$ zeros into the same disk, so every obstacle disk, and the exterior of $C_0$, contains as many zeros as poles of $R$. The winding number of $R$ along every $C_i$ is therefore zero, and $\Psi$ has a continuous single-valued lift on $C_i$. On $C_i$, $|\partial\Psi/\partial\tau_i|=|\partial\Phi/\partial n_i|\le\bar\varepsilon_i$, and any two points of $C_i$ are joined by an arc of length at most $\pi r_i$. Hence $\max_{C_i}\Psi-\min_{C_i}\Psi\le\pi r_i\bar\varepsilon_i$, which proves item~1. For item~2, $\Omega'=\sum_p 1/(z-p)-\sum_\zeta 1/(z-\zeta)$, so $|\partial\Psi/\partial\tau_i|\le|\Omega'|\le L_i$ on $C_i$. Every $z\in C_i$ lies within arc length $h_i/2$ of a sample $z_{i,\ell}$, so $d_{2\pi}(\Psi(z),\theta+\pi)\ge d_{2\pi}(\Psi(z_{i,\ell}),\theta+\pi)-L_ih_i/2>0$. By Lemma~\ref{lem:phase}, no point of $\partial D$ lies on $\mathrm{GRL}(\theta)$.
\end{proof}

Proposition~\ref{prop:certificate} separates what is guaranteed from what is approximated. Non-penetration is not inferred from a small residual: it is certified for $\Omega^{(m)}_{\mathrm{obs}}$ by a finite test, and the residual only measures how many phases the test rejects. Together with Theorem~\ref{thm:main}, an admissible phase that passes \eqref{eq:certificate} yields convergent, mutually disjoint and obstacle-respecting paths, with no assumption on the convergence of the reflection series. The constant $L_i$ is conservative, because it ignores the cancellations between zeros and poles; with many images the test therefore requires many samples per circle.

\subsection{The reciprocal root locus as a special case}\label{subsec:rrl}

The RRL constructions of \cite{Landi,BALESTRINO} use the upper half disk
\begin{equation}\label{eq:Dplus}
D_+:=\{z\in\mathbb{C}:|z|<1,\ \Im z>0\}.
\end{equation}
Let $p_i,\zeta_i\in D_+$ and
\begin{equation}\label{eq:G}
G(z)=\frac{\prod_{i=1}^{N}(z-\zeta_i)(z-\overline{\zeta}_i)}{\prod_{i=1}^{N}(z-p_i)(z-\overline{p}_i)},\qquad R_r(z):=G(z)G(1/z),
\end{equation}
with potential $\Omega_r=-\log R_r=\Phi_r+j\Psi_r$. The conjugate factors are the reflections of starts and goals in the real axis, and $z\mapsto1/z$ composed with conjugation is the reflection in the unit circle: $R_r$ is the Milne--Thomson construction for the two boundary components of $D_+$, which here is exact.

\begin{proposition}\label{prop:rrl}
The unit semicircle and the interval $(-1,1)$ are level sets of $\Psi_r$, with $\Psi_r\equiv0$. Hence $\Theta_\partial=\{-\pi\}$ for $D_+$.
\end{proposition}
\begin{proof}
If $|z|=1$, then $1/z=\overline z$, and by the conjugate symmetry of $G$, $R_r(z)=G(z)G(\overline z)=G(z)\overline{G(z)}=|G(z)|^2>0$. If $x\in(-1,1)$, each conjugate pair in $G$ gives a positive factor, so $G(x)>0$; since $1/x$ is real, $G(1/x)>0$ as well, and $R_r(x)>0$.
\end{proof}

\begin{corollary}\label{cor:rrl}
Every maximal solution of $\dot z=\lambda(z)\nabla\Phi_r(z)^\top$ with $z(0)\in D_+$ remains in $D_+$. If moreover $0\notin\Theta_b$, Theorem~\ref{thm:main} applies to the classical RRL ($\theta=0$): the branches join each start to a distinct goal, are pairwise disjoint and do not touch $\partial D_+$.
\end{corollary}
\begin{proof}
Since $\Psi_r$ is constant on $\partial D_+$, the field $\nabla\Phi_r^\top$ is tangent to $\partial D_+$ wherever it does not vanish. Hence $\partial D_+$ is a union of trajectories of the motion and of finitely many equilibria (the points $\pm1$, where the two pieces of the boundary meet, and the zeros of $\Omega_r'$ on $\partial D_+$). By uniqueness of solutions, a trajectory starting in $D_+$ cannot reach a point of $\partial D_+$ in finite time, so it remains in $D_+$. For the second claim, $R_r$ has $p_i,\zeta_i$ as its only zeros and poles in $\overline{D_+}$, all simple, so Assumption~\ref{ass:domain} holds, and $\theta=0\notin\Theta_\partial=\{-\pi\}$.
\end{proof}

The confinement used geometrically in \cite{Landi,BALESTRINO} is thus a consequence of streamline invariance (Fig.~\ref{fig:rrl}).

\begin{figure}[tb]
\centering
\begin{subfigure}{0.34\textwidth}
    \centering
    \includegraphics[width=\linewidth]{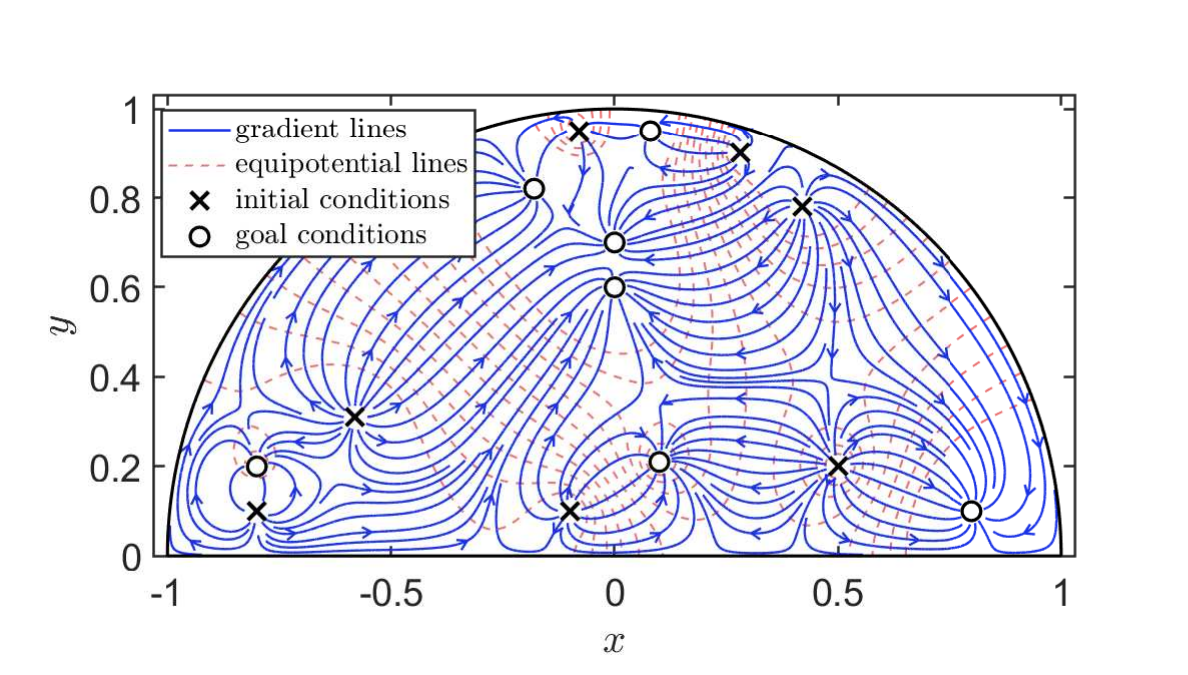}
    \caption{}
\end{subfigure}
\hspace{0.04\textwidth}
\begin{subfigure}{0.34\textwidth}
    \centering
    \includegraphics[width=\linewidth]{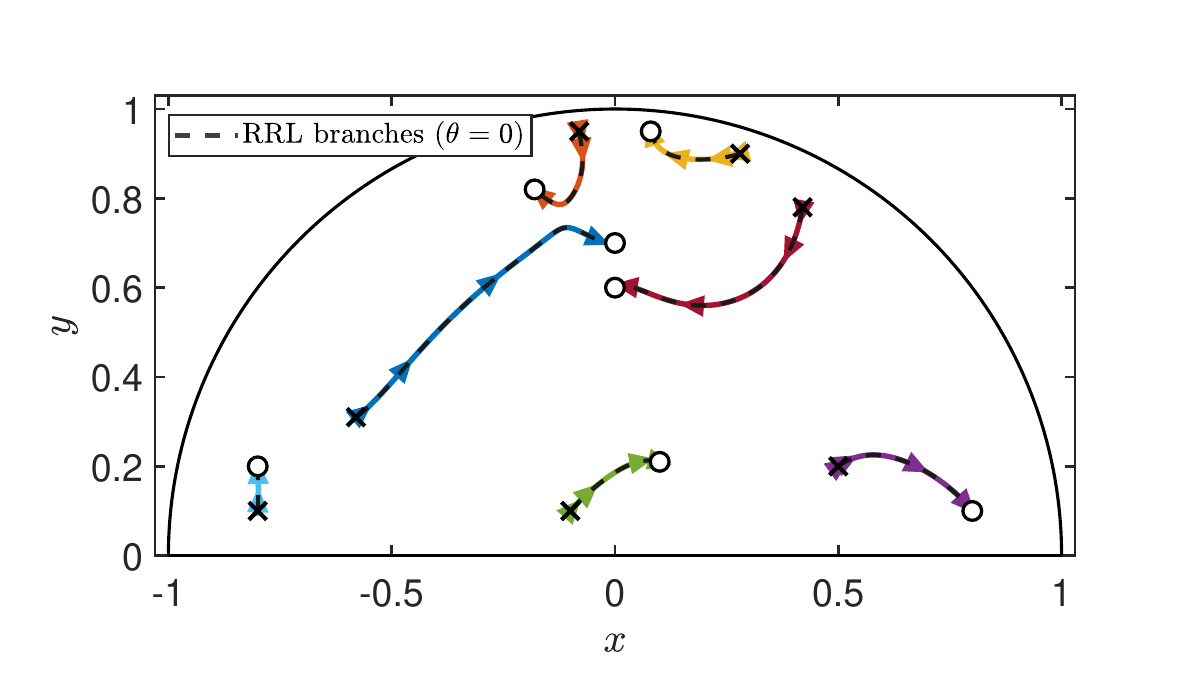}
    \caption{}
\end{subfigure}
\caption{Reciprocal root locus in $D_+$: (a) streamlines and equipotential lines of $\Omega_r$; (b) RRL branches for $\theta=0$ (dashed) and induced trajectories (solid).}
\label{fig:rrl}
\end{figure}

\section{Computational procedure}\label{sec:procedure}

Algorithm~\ref{alg:procedure} summarizes the construction. All quantities are evaluated from the analytic expressions
\begin{equation}\label{eq:derivatives}
\Omega'(z)=\sum_{p}\frac{1}{z-p}-\sum_{\zeta}\frac{1}{z-\zeta},\qquad \Omega''(z)=-\sum_{p}\frac{1}{(z-p)^2}+\sum_{\zeta}\frac{1}{(z-\zeta)^2},
\end{equation}
where the sums run over all zeros and poles of $R^{(m)}_{\mathrm{obs}}$, images included. No grid and no numerical gradient are used.

\paragraph{\texorpdfstring{Critical points and $\Theta_b$}{Critical points and critical phases}} The zeros of $\Omega'$ in $D$ are computed by Newton's method, $z\leftarrow z-\Omega'(z)/\Omega''(z)$, started from a grid of points in $D$. Since a grid search may miss a root, their number is checked with the argument principle. $\Omega'$ has a simple pole at each of the $2N$ zeros and poles of $R$ in $D$ and no other pole there, so the number of its zeros in $D$ is
\begin{equation}\label{eq:argument}
N_{\mathrm{crit}}=2N+\frac{1}{2\pi j}\oint_{\partial D}\frac{\Omega''(z)}{\Omega'(z)}\,dz ,
\end{equation}
with $C_0$ oriented counterclockwise and the obstacle circles clockwise. The principle requires $\Omega'\neq0$ on the contour. When a circle is a streamline, $\Omega'$ typically vanishes at stagnation points on it. A zero $b\in\partial D$ gives the phase $\pi-\arg R(b)$, which belongs to $\Theta_\partial$ and is already excluded; the contour is then moved slightly into $D$, which does not change the count as long as no critical point lies in the thin band between the two contours. If the count \eqref{eq:argument} differs from the number of roots found, the grid is refined. The critical phases follow from \eqref{eq:Theta_b}.

\paragraph{Boundary phases and certificate} The samples $\Psi(z_{i,\ell})$ do not depend on $\theta$, so the test \eqref{eq:certificate} is evaluated on a whole grid of phases at the cost of one evaluation per sample. The selected phase maximizes the margin
\begin{equation}\label{eq:margin}
\mu(\theta):=\min\Bigl\{d_{2\pi}(\theta,\Theta_b),\ \min_{i,\ell}\bigl[d_{2\pi}(\Psi(z_{i,\ell}),\theta+\pi)-L_ih_i/2\bigr]\Bigr\},
\end{equation}
possibly together with a secondary criterion such as path length or clearance.

\paragraph{Branch tracing} Along a branch, $\Omega(z)=\Phi+j(\theta+\pi)$, so the log-gain $\Phi$ is a natural parameter and $dz/d\Phi=1/\Omega'(z)$. Each branch is started at $z=p_i+\epsilon e^{j\beta_i}$, with $\beta_i$ the departure angle \eqref{eq:departure}. A predictor step $z\leftarrow z+\Delta\Phi/\Omega'(z)$ is followed by Newton corrections on $\Omega(z)=\Phi+j(\theta+\pi)$, with the increment of $\Omega$ computed as $-\log(R(z)/R(z_{\mathrm{prev}}))$ to avoid branch cuts. The tracing stops when $z$ is within a tolerance of a zero, which identifies $\sigma_\theta(i)$. The drift of $\Psi$ from $\theta+\pi$ along the traced branches is monitored; in the reference examples of Section~\ref{subsec:reference} it stays below $10^{-9}$. The curves are finally mapped back to the workspace by $f^{-1}$.

\begin{algorithm}[tb]
\caption{Certified GRL planning}\label{alg:procedure}
\begin{algorithmic}[1]
\Require Polygonal workspace $P$ (inflated), starts and goals, reflection depth $m$, phase grid $\mathcal G$
\State Compute the conformal map $f:P\to D$; set $p_i=f(\text{start}_i)$, $\zeta_i=f(\text{goal}_i)$
\State Generate the images \eqref{eq:recursive} up to depth $m$ and form $R^{(m)}_{\mathrm{obs}}$ \eqref{eq:R_obs}
\State Find the zeros $b$ of $\Omega'$ in $D$ by Newton's method from a grid
\State Compute $N_{\mathrm{crit}}$ by \eqref{eq:argument}; refine the grid until the counts agree
\State Form $\Theta_b$ by \eqref{eq:Theta_b}
\For{$i=0,\dots,N_{\mathrm{obs}}-1$}
  \State Compute $\bar\varepsilon_i$, $L_i$ and the lifted samples $\Psi(z_{i,\ell})$, $\ell=1,\dots,M_i$
\EndFor
\State $\mathcal G_{\mathrm{ok}}\gets\{\theta\in\mathcal G:\ \theta\notin\Theta_b\text{ and \eqref{eq:certificate} holds}\}$
\If{$\mathcal G_{\mathrm{ok}}=\emptyset$}
  \State \Return failure (increase $m$ or $M_i$)
\EndIf
\State $\theta^\ast\gets\arg\max_{\theta\in\mathcal G_{\mathrm{ok}}}\mu(\theta)$
\For{$i=1,\dots,N$}
  \State Trace $B_i$ from $p_i$ along \eqref{eq:departure} by predictor--corrector in $\Phi$ until a zero $\zeta_{\sigma(i)}$ is reached
\EndFor
\State Map the branches back by $f^{-1}$; compute path lengths and clearances
\Ensure Paths, assignment $\sigma_{\theta^\ast}$, margin $\mu(\theta^\ast)$; or failure
\end{algorithmic}
\end{algorithm}

The cost of every step is linear in the number $2N+N_I^{(m)}$ of singularities, times the number of evaluation points; the critical-point search dominates, and all steps are parallel.

\section{Numerical results}\label{sec:numerics}

\subsection{Reference examples with polygonal obstacles}\label{subsec:reference}

The \emph{Flagship} example has $N=5$ robots and $N_{\mathrm{obs}}=5$ boundary components, the outer one included (Fig.~\ref{fig:conf_reduction}); the map is normalized with $\alpha=2j$. The \emph{Warehouse} example has $N=3$ robots and $N_{\mathrm{obs}}=7$, with shelves as obstacles; here $\alpha=6.5+3.5j$. In both cases the obstacle-compatible potential is built by reflections up to depth $m$ (Figs.~\ref{fig:flagship}(a) and~\ref{fig:warehouse}(a)).

Table~\ref{tab:certificate} reports the certificate for $m=1,2,3$. In both examples the number of critical points found equals the one given by the argument principle, so $\Theta_b$ is complete, and the winding of $R$ is zero on every circle, as used in Proposition~\ref{prop:certificate}. The certified fraction of the phase circle grows with $m$, consistently with item~1 of the proposition and with the decrease of the residual. The phase $\theta=0$ used in Figs.~\ref{fig:flagship} and~\ref{fig:warehouse} is certified for $m\ge2$ in the Flagship example and for every $m$ in the Warehouse example, so the plotted paths ($m=3$) are covered by Theorem~\ref{thm:main}.

The rejected case is instructive. For the Flagship with $m=1$, the certificate fails for $\theta=0$ by $2\cdot10^{-3}$\,rad, and the traced branch does enter an obstacle, by $0.04$ in the units of $D$. For $m=2$ and $m=3$ the same branch is certified and keeps a clearance of $0.004$ and $0.025$. Here the rejection is not a false alarm: the test detects a configuration in which the truncation is too coarse.

\begin{table}[tb]
\centering
\caption{Certificate on the reference examples: largest residual $\max_i\bar\varepsilon_i$, number of critical points found/predicted by \eqref{eq:argument}, fraction of certified phases, and margin of $\theta=0$ (distance to $\Theta_b\cup\Theta_\partial$).}
\label{tab:certificate}
\begin{tabular}{@{}llcccc@{}}
\toprule
 & $m$ & $\max_i\bar\varepsilon_i$ & crit. & certified & margin of $\theta=0$\\
\midrule
Flagship  & 1 & 3.48 & 6/6 & 27\% & rejected \\
          & 2 & 2.01 & 7/7 & 53\% & 0.04 rad \\
          & 3 & 1.04 & 7/7 & 70\% & 0.33 rad \\
\midrule
Warehouse & 1 & 15.1 & 3/3 & 44\% & 0.30 rad \\
          & 2 & 8.39 & 5/5 & 73\% & 0.23 rad \\
          & 3 & 2.87 & 4/4 & 83\% & 0.22 rad \\
\bottomrule
\end{tabular}
\end{table}

\begin{figure}[tb]
\centering
\begin{subfigure}{0.32\textwidth}\centering\includegraphics[width=\linewidth]{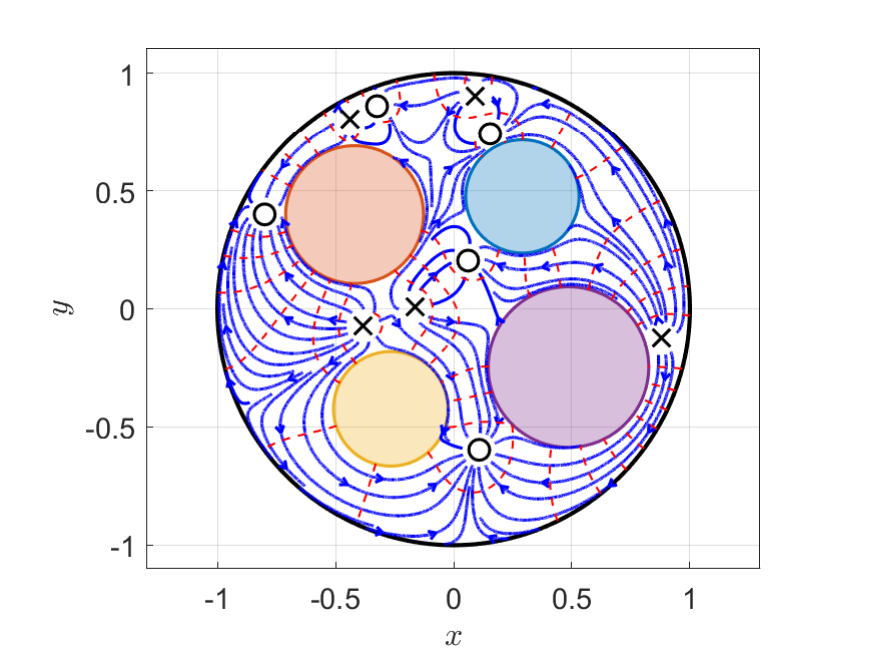}\caption{}\end{subfigure}\hfill
\begin{subfigure}{0.32\textwidth}\centering\includegraphics[width=\linewidth]{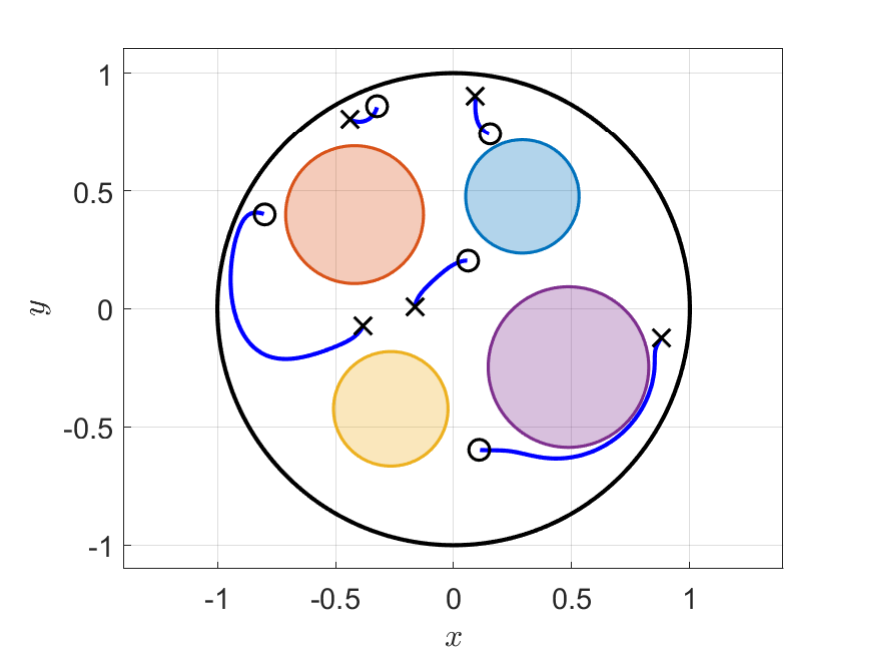}\caption{}\end{subfigure}\hfill
\begin{subfigure}{0.32\textwidth}\centering\includegraphics[width=\linewidth]{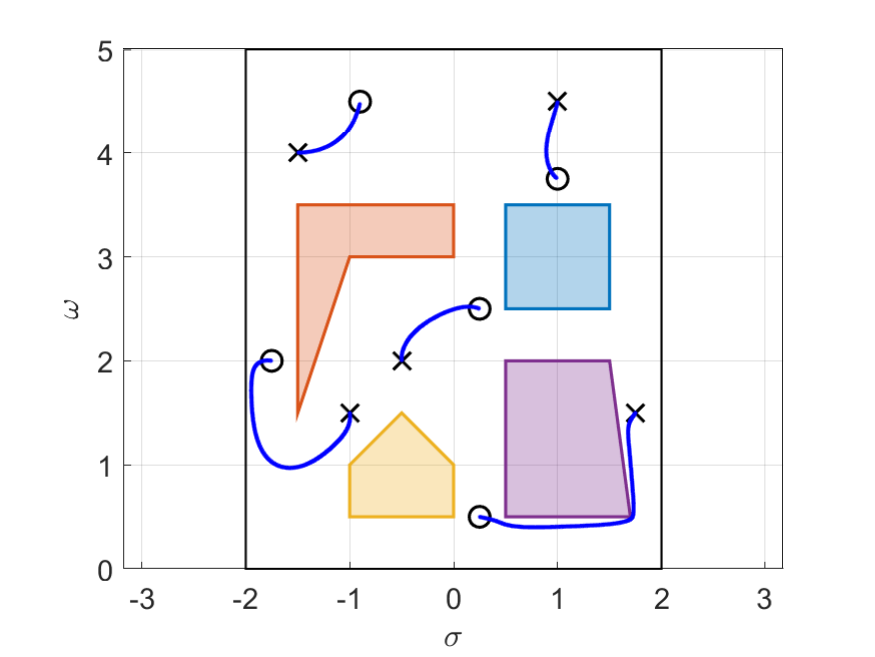}\caption{}\end{subfigure}
\caption{Flagship example, $m=3$: (a) streamlines (blue) and equipotential lines (red) of $\Omega^{(3)}_{\mathrm{obs}}$ in $D$; (b) GRL branches for $\theta=0$; (c) paths mapped back to the workspace. Crosses: starts; circles: goals.}
\label{fig:flagship}
\end{figure}

\begin{figure}[tb]
\centering
\begin{subfigure}{0.32\textwidth}\centering\includegraphics[width=\linewidth]{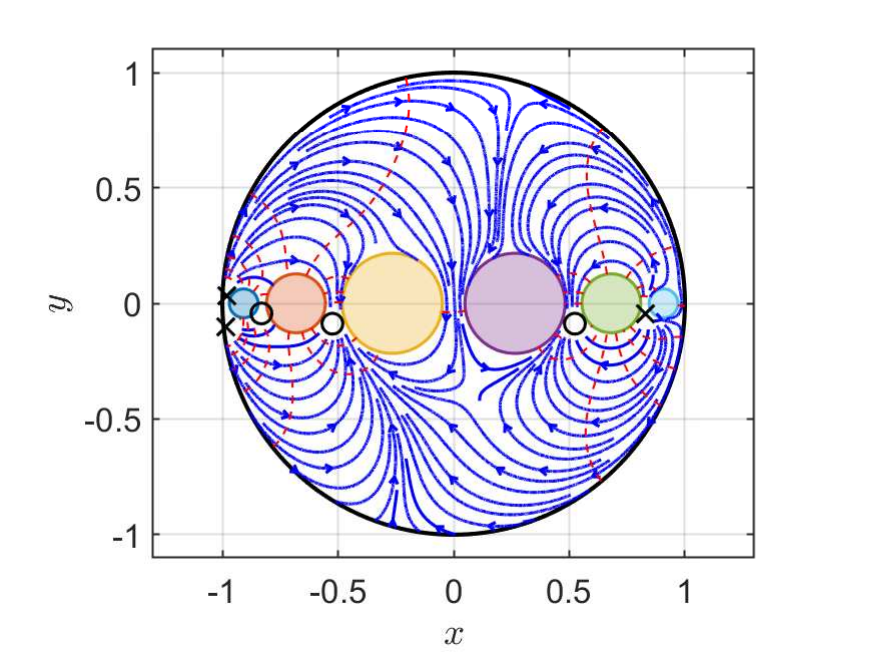}\caption{}\end{subfigure}\hfill
\begin{subfigure}{0.32\textwidth}\centering\includegraphics[width=\linewidth]{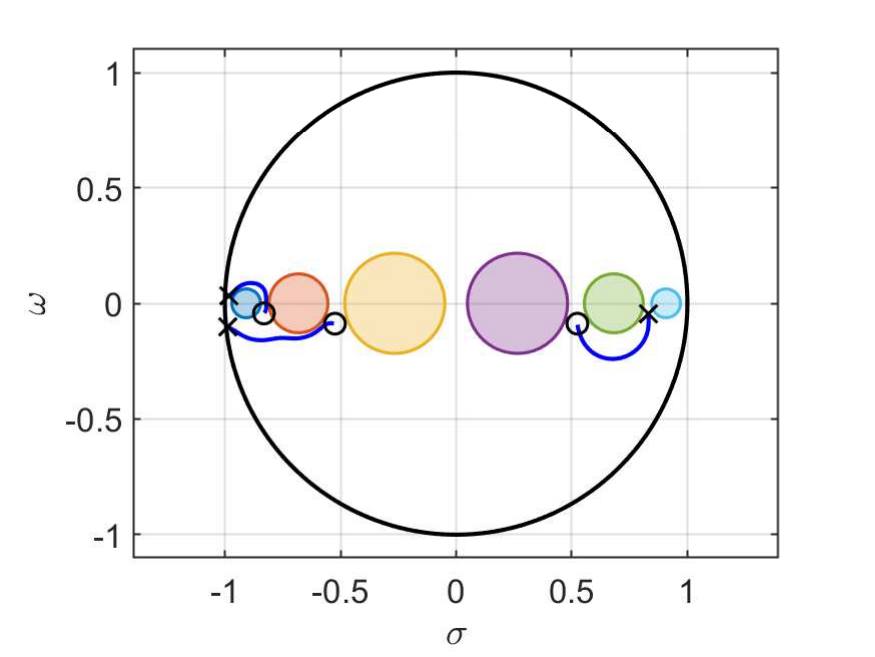}\caption{}\end{subfigure}\hfill
\begin{subfigure}{0.32\textwidth}\centering\includegraphics[width=\linewidth]{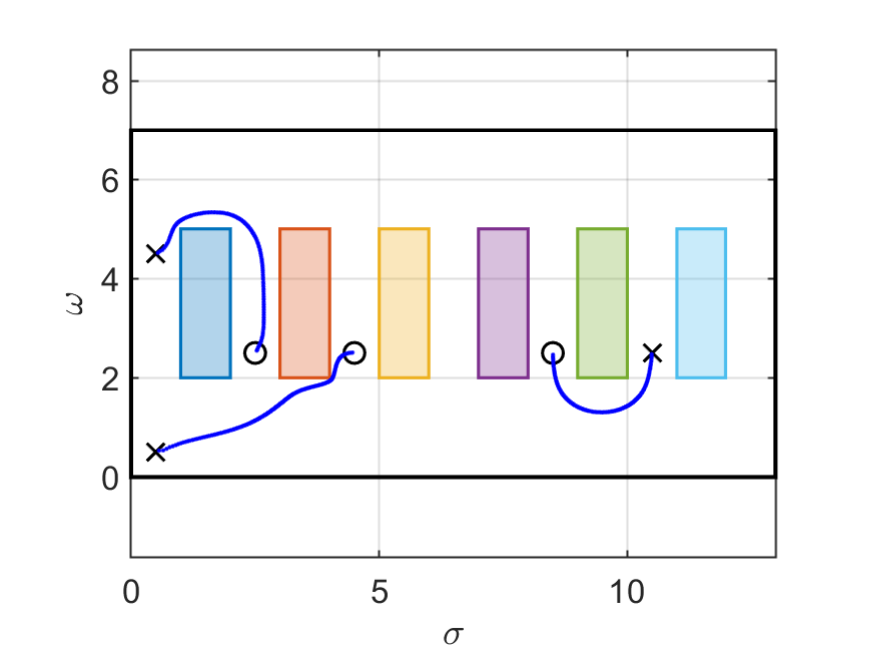}\caption{}\end{subfigure}
\caption{Warehouse example, $m=3$: (a) field of $\Omega^{(3)}_{\mathrm{obs}}$ in $D$; (b) branches for $\theta=0$; (c) paths in the workspace, with the shelves as obstacles.}
\label{fig:warehouse}
\end{figure}

\subsection{Monte Carlo comparison}\label{subsec:mc}

To assess the construction beyond hand-picked configurations, we generated 50 random instances with $N=5$ in each of the two circular domains, with starts and goals at distance at least $0.04$ from the circles and $0.08$ from each other, and $m=2$. The proposed method uses the phase selected by Algorithm~\ref{alg:procedure}. It is compared with two baselines, both using the assignment $\pi$ that minimizes the sum of squared distances:
\begin{itemize}[leftmargin=*,itemsep=2pt,topsep=3pt]
\item \emph{independent HFM}: one harmonic potential per robot (single start--goal pair, same images), each robot following the shortest of eight streamlines (equispaced phases) that reach its goal inside $D$;
\item \emph{APF}: Khatib potentials \cite{Khatib_Potential_} with obstacle and inter-robot repulsion, integrated simultaneously.
\end{itemize}
Results are in Table~\ref{tab:mc}, where $L$ is the total path length and $L_0=\sum_i|p_i-\zeta_{\pi(i)}|$ the straight-line length with the baseline assignment. $L_0$ is the same for the three methods, so the ratios are directly comparable; table entries are medians over the instances.

\begin{table}[tb]
\centering
\caption{Monte Carlo comparison (50 instances per domain, $N=5$, medians).}
\label{tab:mc}
\begin{tabular}{@{}lcccc@{}}
\toprule
 & \multicolumn{2}{c}{Flagship domain} & \multicolumn{2}{c}{Warehouse domain}\\
 & success & $L/L_0$ & success & $L/L_0$\\
\midrule
GRL (certified)      & 100\% & 1.15 & 100\% & 1.13\\
Independent HFM      & 74\%$^\ast$ & 1.06 & 78\%$^\ast$ & 1.03\\
APF                  & 24\% & 1.12 & 6\% & 1.00$^\dagger$\\
\bottomrule
\end{tabular}

\smallskip
{\footnotesize $^\ast$All robots reach their goals; the paths cross in 26\% and 22\% of the instances. $^\dagger$Median over the 3 successful instances.}
\end{table}

All 100 GRL instances were certified: the critical set was complete, every robot reached a distinct goal inside $D$, and the median margin of the selected phase was $0.65$ and $0.52$\,rad. Independent HFM produces shorter paths, but in about one instance out of four two paths cross. Crossing paths do not necessarily collide, since this depends on the timing, but they rule out the timing-independent safety given by item~3 of Theorem~\ref{thm:main}. APF fails mostly because of local minima behind the obstacles. The price of the certified construction is a moderate increase of path length, about $10\%$ with respect to independent HFM (ratio of the median values of $L/L_0$). The median computation time per instance is comparable for the two harmonic methods (18--42\,s in an interpreted, single-thread Octave implementation), and is dominated by the critical-point search.

\subsection{Forest: ten robots and ten obstacles}\label{subsec:forest}

The last example tests the construction on a denser scene, directly in a circular domain. Ten robots cross a ``forest'' of ten circular pillars in the unit disk (Table~\ref{tab:forest}). The starts are $p_i=0.84\,e^{j\alpha_i}$ with $\alpha_i$ equally spaced from $125^\circ$ to $235^\circ$, and the goals are $\zeta_\ell=0.84\,e^{j\beta_\ell}$ with $\beta_\ell$ equally spaced from $-55^\circ$ to $55^\circ$; the robots must therefore pass through the obstacles to reach the other side. With $N_{\mathrm{obs}}=11$ and depth $m=3$, \eqref{eq:NI} gives $N_I^{(3)}=24\,420$ image singularities.

\begin{table}[tb]
\centering
\caption{Forest example: centers and radii of the obstacles in the unit disk.}
\label{tab:forest}
\small
\begin{tabular}{@{}cccccc@{}}
\toprule
$i$ & $c_i$ & $r_i$ & $i$ & $c_i$ & $r_i$\\
\midrule
1 & $-0.414-0.326j$ & 0.092 & 6  & $\phantom{-}0.392+0.106j$ & 0.079\\
2 & $\phantom{-}0.082-0.503j$ & 0.077 & 7  & $-0.281+0.409j$ & 0.086\\
3 & $-0.069+0.108j$ & 0.090 & 8  & $\phantom{-}0.183+0.397j$ & 0.077\\
4 & $\phantom{-}0.456-0.268j$ & 0.086 & 9  & $-0.375-0.024j$ & 0.081\\
5 & $\phantom{-}0.196-0.257j$ & 0.060 & 10 & $-0.111-0.148j$ & 0.096\\
\bottomrule
\end{tabular}
\end{table}

Figure~\ref{fig:forest}(a) shows the resulting field. The critical-point search finds 18 zeros of $\Omega'$ in $D$, and the argument principle \eqref{eq:argument} gives the same number, so $\Theta_b$ consists of 18 phases. The residual is $0.16$ on the outer circle and between $0.81$ and $1.65$ on the obstacles. The boundary phases were evaluated on $2\,232$ to $9\,855$ samples per circle (Fig.~\ref{fig:forest}(c)): each circle removes an arc of $6.4^\circ$ to $18.0^\circ$, within the bound $\pi r_i\bar\varepsilon_i$ of Proposition~\ref{prop:certificate}, and $\Theta_\partial$ covers $120^\circ$ in total. About $69\%$ of the phase circle is admissible. With this many images the constant $L_i$ in \eqref{eq:certificate} is large, so we report the sampled boundary phases, whose spacing is much smaller than the margins involved.

The selected phase is $\theta^\ast=106.2^\circ$, at $51.7^\circ$ from $\Theta_\partial$ and $25.8^\circ$ from $\Theta_b$. Figure~\ref{fig:forest}(b) shows the ten branches. They stay inside $D$, are pairwise disjoint, and end at ten distinct goals with $\sigma(i)=11-i$: each robot reaches the goal at the same height on the opposite side, and the paths form lanes between the pillars. The markers show the positions of the robots at four common values of the gain, as in item~4 of Theorem~\ref{thm:main}: at every instant the robots are the ten closed-loop roots of $1+ke^{j\theta^\ast}R=0$ in $D$. The minimum distance between two paths is $0.066$ and the minimum distance from the obstacles and the wall is $0.011$, in the units of $D$. The total length is $15.35$, against $13.82$ for straight segments with the same assignment, a ratio of $1.11$, in line with the Monte Carlo results.

\begin{figure}[tb]
\centering
\begin{subfigure}[b]{0.32\textwidth}\centering\includegraphics[width=\linewidth]{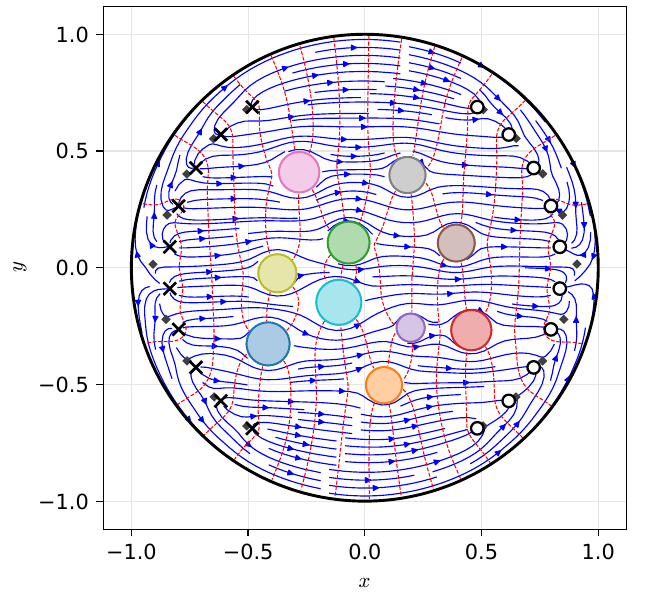}\caption{}\end{subfigure}\hfill
\begin{subfigure}[b]{0.32\textwidth}\centering\includegraphics[width=\linewidth]{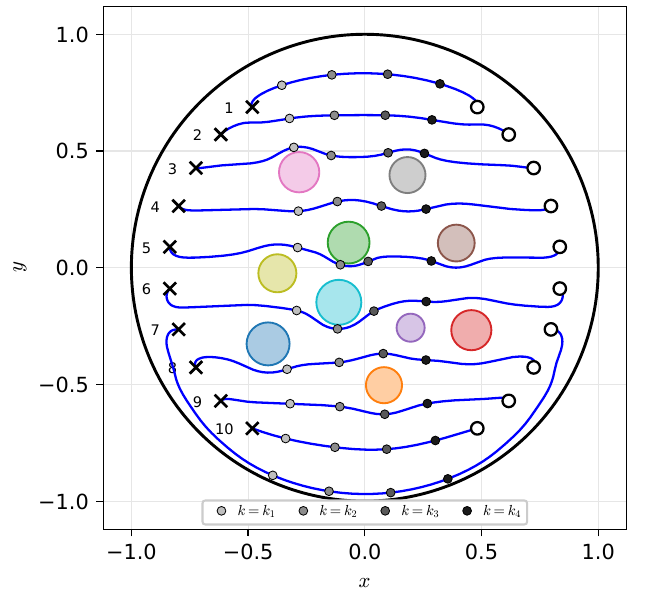}\caption{}\end{subfigure}\hfill
\begin{subfigure}[b]{0.33\textwidth}\centering\includegraphics[width=\linewidth]{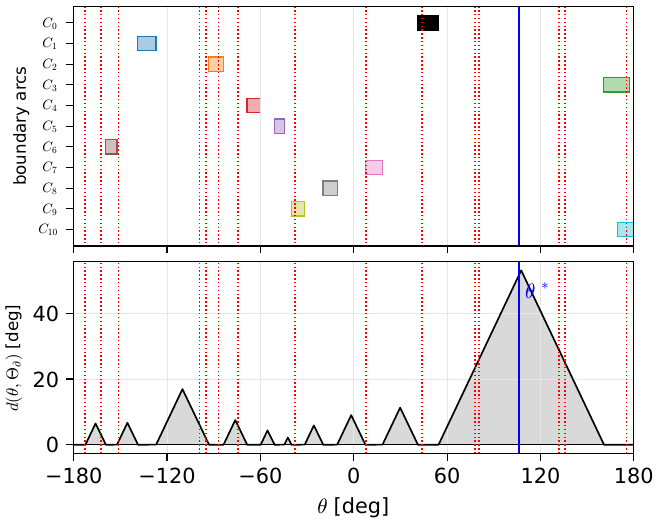}\caption{}\end{subfigure}
\caption{Forest example. (a) Streamlines (blue) and equipotential lines (red, dashed) of $\Omega^{(3)}_{\mathrm{obs}}$; grey diamonds: the 18 critical points. (b) Branches of $\mathrm{GRL}(\theta^\ast)$, $\theta^\ast=106.2^\circ$, and positions of the robots at four common gains $k_1<k_2<k_3<k_4$; robot $i$ reaches $\zeta_{11-i}$. (c) Top: arc of $\Theta_\partial$ contributed by each circle ($C_0$ is the outer wall); bottom: distance of $\theta$ from $\Theta_\partial$; red dotted lines: critical phases; blue line: $\theta^\ast$.}
\label{fig:forest}
\end{figure}

As a further check, the branches were traced for all 47 admissible phases of a $5^\circ$ grid: in every case they stayed in $D$, were disjoint and ended at distinct goals, and the four assignments observed changed only across critical phases (Remark~\ref{rem:sigma}).

\section{Conclusions}\label{sec:conclusions}

We have shown that the GRL of a rational function encoding starts and goals, read through its logarithmic potential after a conformal map, provides certified coordinated paths for $N$ interchangeable robots in polygonal environments: for every admissible phase the robots reach distinct goals along disjoint paths, with an assignment produced by the construction and a time law given by the gain. For potentials computed with finitely many reflections, a finite test certifies obstacle safety, and the RRL is recovered as a special case.

The method is offline and geometric: tracking, dynamics and online replanning are left to the execution layer, robot size enters through obstacle inflation, and optimality of the paths is not claimed. The main open directions are faster series-based obstacle potentials \cite{DELILLO_pot_flow_series,trefethen_series_solution} for larger scenes, a sharper certificate exploiting the cancellations between image zeros and poles, a systematic study of the dependence on the phase and on the normalization $\alpha$, and the coupling with an online layer.

\section*{Code availability}\label{sec:code}

The MATLAB implementation, the verification scripts and the code of the Forest example are available at \url{https://github.com/mfrascarelli/grl-certified-planning}. An interactive companion notebook, which reproduces the main objects of the paper on small examples, runs in the browser at \url{https://mfrascarelli.github.io/grl-certified-planning/}.

\section*{Acknowledgments}

This work was supported by the Italian Ministry for University and Research (MUR) in the framework of the FoReLab project (Departments of Excellence).

\renewcommand{\bibfont}{\footnotesize}\setlength{\bibsep}{1pt}
\bibliographystyle{elsarticle-num}
\bibliography{GRL_arxiv}

\end{document}